\documentclass[twoside,11pt]{article}

\usepackage{blindtext}

\usepackage[abbrvbib, preprint]{jmlr2e}

\usepackage{lastpage}
\usepackage{graphicx}
\usepackage{subfigure}
\usepackage{algpseudocode}
\usepackage{booktabs} 
\usepackage{tikz} 
\usepackage{algorithm}
\usepackage{amssymb}
\usepackage{amsmath}
\usepackage{mathrsfs}
\usepackage{array, booktabs}
\usepackage{wrapfig}
\usepackage{url}
\usepackage{tabularx}
\usepackage{multirow}

\usepackage[colorlinks,
linkcolor=red,
anchorcolor=blue,
citecolor=blue]{hyperref}

\newcommand{\BibTeX}{\textsc{B\kern-0.1emi\kern-0.017emb}\kern-0.15em\TeX}
\def\ci{\perp\!\!\!\perp}
\def\nni{\not\!\perp\!\!\!\perp}

\usepackage[symbol]{footmisc}

\ShortHeadings{Functional Linear Non-Gaussian Acyclic Model for Causal Discovery}{Yang, Lee, Zhang, Suzuki}
\firstpageno{1}
\begin{document}


\title{Functional Linear Non-Gaussian Acyclic Model for Causal Discovery}

\author{\name Tian-Le Yang \email yangtianle1996@gmail.com \\
       \addr Graduate School of Engineering Science\\
       Osaka University\\ 
       1-3, Machikaneyama, Toyonaka, Osaka, Japan
       \AND
       \name Kuang-Yao Lee \email kuang-yao.lee@temple.edu\\
       \addr Department of Statistics, Operations, and Data Science\\ Temple University\\ Philadelphia, PA 19122 United States
       \AND
       \name Kun Zhang \email kunz1@cmu.edu\\
       \addr Machine Learning Department \\ Carnegie Mellon University \& Mohamed bin Zayed University of Artificial Intelligence \\
       Pittsburgh, PA 15213, USA \& Masdar city, Abu Dhabi, United Arab Emirates
       \AND
       \name Joe Suzuki
       \email prof.joe.suzuki@gmail.com \\
       \addr Graduate School of Engineering Science\\
       Osaka University\\ 
       1-3, Machikaneyama, Toyonaka, Osaka, Japan
       }

\maketitle

\begin{abstract}
In causal discovery, non-Gaussianity has been used to characterize the complete configuration of a Linear Non-Gaussian Acyclic Model (LiNGAM), encompassing both the causal ordering of variables and their respective connection strengths. 
However, LiNGAM can only deal with the finite-dimensional case. To expand this concept, we extend the notion of variables to encompass vectors and even functions, leading to the Functional Linear Non-Gaussian Acyclic Model (Func-LiNGAM). Our motivation stems from the desire to identify causal relationships in brain-effective connectivity tasks involving, for example, fMRI and EEG datasets. We demonstrate why the original LiNGAM fails to handle these inherently infinite-dimensional datasets and explain the availability of functional data analysis from both empirical and theoretical perspectives. {We establish theoretical guarantees of the identifiability of the causal relationship among non-Gaussian random vectors and even random functions in infinite-dimensional Hilbert spaces.} To address the issue of sparsity in discrete time points within intrinsic infinite-dimensional functional data, we propose optimizing the coordinates of the vectors using functional principal component analysis.
Experimental results on synthetic data verify the ability of the proposed framework to identify causal relationships among multivariate functions using the observed samples. For real data, we focus on analyzing the brain connectivity patterns derived from fMRI data.
\end{abstract}


\section{Introduction}\label{s1}
Numerous empirical sciences strive to uncover and comprehend causal mechanisms that underlie a wide range of natural phenomena and human social behavior. Causal discovery has a wide range of applications, including biology \citep{biology}, climate studies \citep{climate}, and healthcare \citep{health}. When determining the cause-and-effect relationship between variables, such as $X_1$ and $X_2$, detecting their dependence alone is insufficient for determining the causal direction, i.e., whether $X_1\rightarrow X_2$ or $X_2\rightarrow X_1$. 

Causal analysis based on the LiNGAM, proposed by \cite{Shimizu}, addresses this challenge by identifying the causal directions in linear relationships. Specifically, supposing there is no latent common cause for $X_1$ and $X_2$, it figures out the causal direction between them by checking which of the following two models holds: $X_2=aX_1+\epsilon$ and $X_1=a'X_2+\epsilon'$, where $X_1\ci \epsilon$ and $X_2\ci \epsilon'$ and $a, a'\in\mathbb{R}$.\footnote{$X_1\ci X_2$ denotes the independence of $X$ and $Y$.} The sufficient and necessary condition of the identifiability is that LiNGAM requires at most one of the noise terms (including the root causes) to be non-Gaussian to make it possible to identify unique causal directions. Notably, zero correlation is synonymous with independence in Gaussian variables, making it impossible to distinguish between the two causal models when $X_1$ and $X_2$ are Gaussian. 

In this linear, Gaussian case, one can only end up with the so-called Markov equivalence class (all members of the equivalence class have the same conditional independence relations), even when adhering to faithfulness assumption \citep{Spirtes,Pearl}. 
For instance,
the three Directed Acyclic Graphs (DAGs) connecting three variables, such as $X_1, X_2, X_3$, in Fig. \ref{f1} are Markov equivalent because they have the same distribution in the Gaussian case.
Here faithfulness refers to the property where any independence relations observed in the data can be explained by the causal relationships represented in the graphical model. However, this is not the case anymore in non-Gaussian cases.  Due to this significant advancement, LiNGAM can uniquely determine the causal ordering among variables solely based on observational data, even without assuming faithfulness. 

For the converse, the Darmois-Skitovich theorem (D-S) is employed to prove the identifiability of causal direction. From D-S, if at least one of the variables $X_{1}$ and $X_{2}$ are non-Gaussian, then only one unique direction of $X_1\rightarrow X_2$ and $X_2\rightarrow X_1$ exists. The Darmois-Skitovich (D-S) theorem originally focused on one-dimensional Gaussian random variables. Interestingly, \cite{MDS} expanded its application to random vectors, while \cite{DSB} generalized it to Banach spaces. In our paper, random elements that take values in a Banach space are called random functions. 

\begin{figure}\label{f1}
\setlength{\unitlength}{0.7mm}
\begin{picture}(210,40)
\put(10,10){\circle{10}}
\put(30,10){\circle{10}}
\put(20,30){\circle{10}}
\put(10,8){\makebox[0pt][c]{$X_1$}}
\put(30,8){\makebox[0pt][c]{$X_2$}}
\put(20,28){\makebox[0pt][c]{$X_3$}}
\put(10,15){\vector(2,3){7}}
\put(22,25){\vector(2,-3){7}}
\put(50,10){\circle{10}}
\put(70,10){\circle{10}}
\put(60,30){\circle{10}}
\put(50,8){\makebox[0pt][c]{$X_1$}}
\put(70,8){\makebox[0pt][c]{$X_2$}}
\put(60,28){\makebox[0pt][c]{$X_3$}}
\put(58,25){\vector(-2,-3){7}}
\put(70,15){\vector(-2,3){7}}
\put(90,10){\circle{10}}
\put(110,10){\circle{10}}
\put(100,30){\circle{10}}
\put(90,8){\makebox[0pt][c]{$X_1$}}
\put(110,8){\makebox[0pt][c]{$X_2$}}
\put(100,28){\makebox[0pt][c]{$X_3$}}
\put(97,26){\vector(-2,-3){7}}
\put(103,26){\vector(2,-3){7}}

\put(150,22){$X_2$}
\put(150,3){$X_1$}
\put(168,16){$\epsilon$}
\put(130,30){$X_1\rightarrow X_2$}
\put(135,10){\vector(1,0){30}}
\put(165,10){\vector(0,1){15}}
\put(135,10){\vector(2,1){30}}

\put(190,22){$X_1$}
\put(190,3){$X_2$}
\put(208,16){$\epsilon^\prime$}
\put(170,30){$X_2\rightarrow X_1$}
\put(175,10){\vector(1,0){30}}
\put(205,10){\vector(0,1){15}}
\put(175,10){\vector(2,1){30}}
\end{picture}
\caption{\small Structure learning methods like the PC algorithm cannot distinguish these causal graphs that have the identical probability distribution $P(X_1X_2)P(X_2X_3)/P(X_2)$ (Left). But LiNGAM can differentiate them via the non-Gaussian assumption (Right).}
\end{figure}
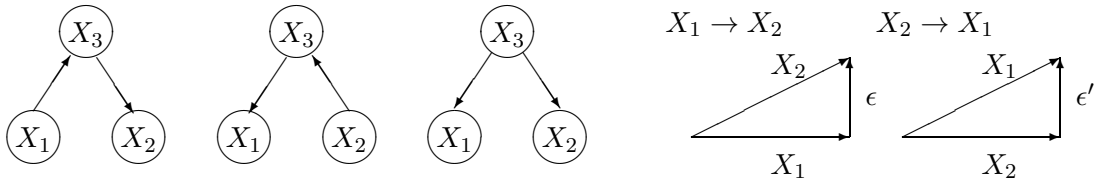

This paper establishes a novel functional framework for modeling the causal structure of multivariate functional data, which is the realization of random functions. It is important to note that functional data is inherently infinite-dimensional. If we apply conventional models such as PC or LiNGAM directly, we might incorrectly identify causal relationships, as shown in Fig. \ref{E1}. To demonstrate the benefits of {functional data analysis \citep{FDA2005}}, we provide an example in {Fig. \ref{E2}}, illustrating how smoothing the discrete points enables us to capture missing information. Functional data analysis has gained prominence in diverse fields, including neuroimaging \citep{22}, finance \citep{21}, and genetics \citep{20}. Exploring causal relationships among random functions presents a significant challenge in multivariate functional data analysis.

This research is motivated by brain-effective connectivity \citep{23}, which explores the directional effects between neural systems. Learning brain-effective connectivity networks from electroencephalogram (EEG), functional magnetic resonance imaging (fMRI), and electrocorticographic imaging (ECoG) records is crucial for understanding brain activities and neuron responses. Modeling these multivariate processes and accurately estimating effective connections between brain areas pose significant challenges due to the continuous nature of the data and the need to treat the data as functions, considering the small time intervals between adjacent sample points.
Previous studies, such as \cite{FGM}, has explored the functional aspects of the Gaussian graphical model by estimating the inverse covariance matrix.  \cite{24} introduced the functional directional relationships under Gaussian assumption, enabling the determination of a directed acyclic graph (DAG) up to its equivalence class. The previous version of this paper \cite{TJ2022} discussed the identifiability without considering one important point for functional data: the covariance operator's non-invertibility. Moreover, the previous algorithm for functional data is not accurate because it only tests the independence of every principal component rather than the whole random vector. {\cite{ZY2022} developed a
novel Bayesian network model for multivariate functional data.} \cite{roy2023directed} considers the directed cyclic model for functional data. In contrast to previous works, our approach differs in that we first establish the identifiability of random vectors. Subsequently, we demonstrate the identifiability of random functions considering the non-invertibility {and extend it into multivariate scenarios.}
\textbf{Our contributions are as follows}:
\begin{itemize}
    \item[$\bullet$] We establish a framework for discovering causal orders for random vectors and functions, moving beyond the traditional focus on random variables.
    \item[$\bullet$] We theoretically prove that it is possible to identify the causal order under non-Gaussianity for random vectors (Theorem \ref{t1}).
    \item[$\bullet$] We further demonstrate the identifiability of the causal order for non-Gaussian processes in infinite-dimensional Hilbert spaces (Theorem \ref{t6}).
    \item[$\bullet$] To verify the validity of our method, we performed extensive experiments
with simulated data as Table~\ref{tab:my_label}. Empirical results demonstrate the identifiability. The results show that it performs worse as the number of functions increases, which is reasonable. But as the sample size increases, it performs better. We need more data for larger dimensions, but the required amounts are still reasonable.
\end{itemize}
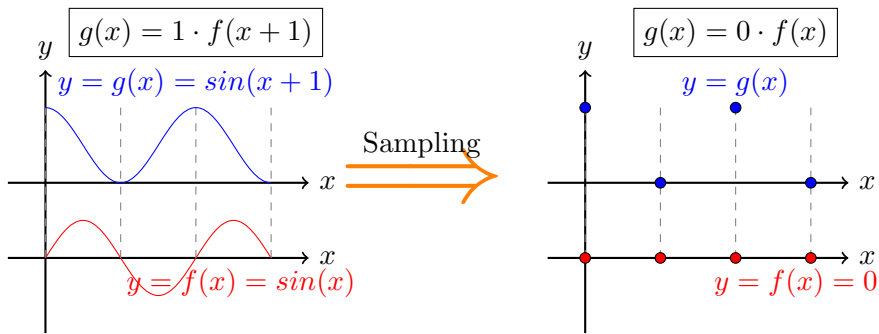
\begin{figure*}[t]
\centering
\begin{tikzpicture}

\draw[->,thick] (-0.5,0) -- (3.5,0) node[right] {$x$};
\draw[->,thick] (-0.5,1) -- (3.5,1) node[right] {$x$};
\draw[->,thick] (0,-1.0) -- (0,2.5) node[above] {$y$};



\draw[color=red] (0,0) sin (0.5,0.5) cos (1,0) sin (1.5,-0.5) cos (2,0) sin (2.5,0.5) cos (3,0);
\node[above, blue] at (2.0,2) {$y = g(x)=sin(x+1)$};
\draw[gray, dashed] (2,0)--(2,2);
\draw[gray, dashed] (0,0)--(0,2);
\draw[gray, dashed] (1,0)--(1,2);
\draw[gray, dashed] (3,0)--(3,2);



\draw[color=blue] (0,2) sin (0,2) cos (0.5,1.5) sin (1,1) cos (1.5,1.5) sin (2,2) cos (2.5,1.5) sin(3,1);

\node[below, red] at (2.6,0) {$y = f(x)=sin(x)$};
\node[draw] at (2.0,3.0) {$g(x)=1\cdot f(x+1)$};
\node[scale=3.5,orange] at (5.0,1.0) {${\Longrightarrow}$};
\node[] at (5.0,1.5) {Sampling};
\end{tikzpicture}
\begin{tikzpicture}
    \node[draw] at (2.0,3.0) {$g(x)=0\cdot f(x)$};
\draw[->,thick] (-0.5,0) -- (3.5,0) node[right] {$x$};
\draw[->,thick] (-0.5,1) -- (3.5,1) node[right] {$x$};
\draw[->,thick] (0,-1.0) -- (0,2.5) node[above] {$y$};
\draw[color=red, dashed] (0,0) (1,0) (2,0) (3,0);
\node[below, red] at (2.8,0) {$y = f(x) = 0$};

\draw[gray, dashed] (2,0)--(2,2);
\draw[gray, dashed] (0,0)--(0,2);
\draw[gray, dashed] (1,0)--(1,2);
\draw[gray, dashed] (3,0)--(3,2);

\node[circle,draw=black, fill=blue, inner sep=0pt,minimum size=4pt] (b) at (0,2) {};
\node[circle,draw=black, fill=blue, inner sep=0pt,minimum size=4pt] (b) at (1,1) {};
\node[circle,draw=black, fill=blue, inner sep=0pt,minimum size=4pt] (b) at (2,2) {};
\node[circle,draw=black, fill=blue, inner sep=0pt,minimum size=4pt] (b) at (3,1) {};

\node[circle,draw=black, fill={rgb:red,1;green,0;blue,0}, inner sep=0pt,minimum size=4pt] (b) at (0,0) {};
\node[circle,draw=black, fill={rgb:red,1;green,0;blue,0}, inner sep=0pt,minimum size=4pt] (b) at (1,0) {};
\node[circle,draw=black, fill={rgb:red,1;green,0;blue,0}, inner sep=0pt,minimum size=4pt] (b) at (2,0) {};
\node[circle,draw=black, fill={rgb:red,1;green,0;blue,0}, inner sep=0pt,minimum size=4pt] (b) at (3,0) {};

\draw[color=blue] (0,2) (1,1) (2,2) (3,1);

\node[above, blue] at (2.0,2) {$y=g(x)$};
\end{tikzpicture}

    \caption{\textbf{Illustration of why the original LiNGAM does not work}. The Left Graph: original two stochastic processes with their causal relationships; The Right Graph: a possible situation where we sample the time series but miss the causal relationship.}
    \label{E1}
\end{figure*} 

The structure of the paper is as follows. Section \ref{s2} provides the necessary background information to comprehend this paper. This includes introducing the LiNGAM, infinite-dimensional Hilbert spaces, and random elements (random functions). Section \ref{s3} and \ref{s4} present the main theoretical results extending the LiNGAM and outlines the corresponding procedure. Section \ref{s5} and \ref{s6} present the experimental results. Section \ref{s7} summarizes the key points.

\section{Background}\label{s2}
\subsection{Linear Non-Gaussian Acyclic Model (LiNGAM)}
This section introduces the concept of the LiNGAM for inferring the causal relationships among random variables.

Suppose two random variables $X_1, X_2\in\mathbb{R}$, we want to identify the causal directions of either $X_1\rightarrow X_2$ or $X_2\rightarrow X_1$.
More specifically, our analysis assumes that $X_1$ and $X_2$ are linearly related and have zero means. Such as{
\begin{equation}\label{eq11}
X_1 = e_1,\quad X_2=aX_1+e_2\ ,
\end{equation}
\begin{equation}\label{eq22}
X_2 = e_1',\quad X_1=a'X_2+e_2'
\end{equation}}
with $a,a'\in {\mathbb R}$ and ${\mathbb E}[\epsilon]={\mathbb E}[\epsilon']=0$.
To be simple, we let 
\begin{equation}\label{eq23}
a\not=0\ ,\ {\rm or}\ a'\not=0\ ,
\end{equation}
to avoid $X_1\ci X_2$.
Specifically, in the context of LiNGAM, under the assumption of the noise terms, denoted as $\epsilon$ and $\epsilon'$, are independent of their respective covariates, $X_1$ and $X_2$ in (\ref{eq11}) and (\ref{eq22}).
Therefore, based on the condition of $X_1\ci e_2$ or $X_2\ci e_2'$, we determine the true causal model to be either (\ref{eq11}) or (\ref{eq22}).
\begin{figure*}
    \centering
\begin{tikzpicture}
\draw[->,thick] (-0.5,0) -- (3.5,0) node[right] {$x$};
\draw[->,thick] (-0.5,1) -- (3.5,1) node[right] {$x$};
\draw[->,thick] (0,-1.0) -- (0,2.5) node[above] {$y$};
\node[below, red] at (2.8,0) {$y = f(x) = 0$};

\draw[gray, dashed] (2,0)--(2,2);
\draw[gray, dashed] (0,0)--(0,2);
\draw[gray, dashed] (1,0)--(1,2);
\draw[gray, dashed] (3,0)--(3,2);

\node[circle,draw=black, fill=blue, inner sep=0pt,minimum size=4pt] (b) at (0,2) {};
\node[circle,draw=black, fill=blue, inner sep=0pt,minimum size=4pt] (b) at (1,1) {};
\node[circle,draw=black, fill=blue, inner sep=0pt,minimum size=4pt] (b) at (2,2) {};
\node[circle,draw=black, fill=blue, inner sep=0pt,minimum size=4pt] (b) at (3,1) {};

\node[circle,draw=black, fill={rgb:red,1;green,0;blue,0}, inner sep=0pt,minimum size=4pt] (b) at (0,0) {};
\node[circle,draw=black, fill={rgb:red,1;green,0;blue,0}, inner sep=0pt,minimum size=4pt] (b) at (1,0) {};
\node[circle,draw=black, fill={rgb:red,1;green,0;blue,0}, inner sep=0pt,minimum size=4pt] (b) at (2,0) {};
\node[circle,draw=black, fill={rgb:red,1;green,0;blue,0}, inner sep=0pt,minimum size=4pt] (b) at (3,0) {};


\node[above, blue] at (2.0,2) {$y=g(x)$};
\node[above, blue] at (2.0,2) {$y=g(x)$};
\node[scale=3.5,orange] at (5.0,1.0) {${\Longrightarrow}$};
\node[] at (5.0,1.5) {FPCA};
\node[draw] at (2.0,3.0) {$g(x)=0\cdot f(x)$};
\end{tikzpicture}
\begin{tikzpicture}
\draw[->,thick] (-0.5,0) -- (3.5,0) node[right] {$x$};
\draw[->,thick] (-0.5,1) -- (3.5,1) node[right] {$x$};
\draw[->,thick] (0,-1.0) -- (0,2.5) node[above] {$y$};
\draw[color=red, dashed] (0,0) sin (0.5,0.5) cos (1,0) sin (1.5,-0.5) cos (2,0) sin (2.5,0.5) cos (3,0);
\node[below, red] at (2.8,0) {$y = f(x) = 0$};

\draw[gray, dashed] (2,0)--(2,2);
\draw[gray, dashed] (0,0)--(0,2);
\draw[gray, dashed] (1,0)--(1,2);
\draw[gray, dashed] (3,0)--(3,2);

\node[circle,draw=black, fill=blue, inner sep=0pt,minimum size=4pt] (b) at (0,2) {};
\node[circle,draw=black, fill=blue, inner sep=0pt,minimum size=4pt] (b) at (1,1) {};
\node[circle,draw=black, fill=blue, inner sep=0pt,minimum size=4pt] (b) at (2,2) {};
\node[circle,draw=black, fill=blue, inner sep=0pt,minimum size=4pt] (b) at (3,1) {};

\node[circle,draw=black, fill={rgb:red,1;green,0;blue,0}, inner sep=0pt,minimum size=4pt] (b) at (0,0) {};
\node[circle,draw=black, fill={rgb:red,1;green,0;blue,0}, inner sep=0pt,minimum size=4pt] (b) at (1,0) {};
\node[circle,draw=black, fill={rgb:red,1;green,0;blue,0}, inner sep=0pt,minimum size=4pt] (b) at (2,0) {};
\node[circle,draw=black, fill={rgb:red,1;green,0;blue,0}, inner sep=0pt,minimum size=4pt] (b) at (3,0) {};

\draw[color=blue, dashed] (0,2) sin (0,2) cos (0.5,1.5) sin (1,1) cos (1.5,1.5) sin (2,2) cos (2.5,1.5) sin(3,1);
\node[draw] at (2.0,3.0) {$g(x)\approx f(x+1)$};
\end{tikzpicture}

    \caption{\textbf{Illustration of Why Func-LiNGAM Work}. (Smoothing: Functional data analysis) The Left Graph: with the worst situation when we sample the time series and miss the causal relationship, where we get $g$ and $f$ have no causal relationship. The Right Graph: we can complete the discrete points into smooth curves with the Functional data analysis, capturing extra information when choosing suitable bases.}
    \label{E2}
\end{figure*}
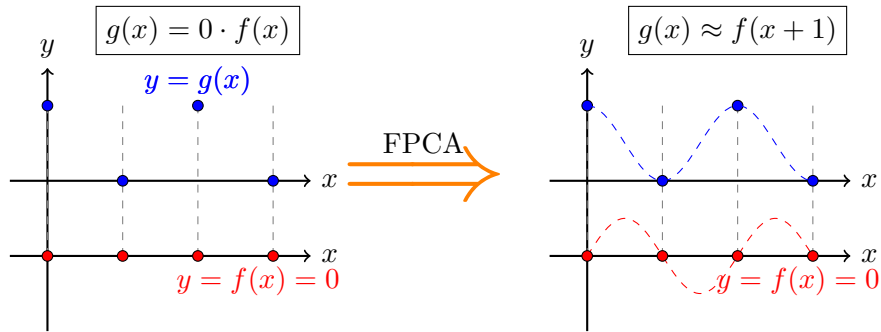
It may initially appear that distinguishing between (\ref{eq11}) and (\ref{eq22}) is not possible, in other words, $X_1$ and $X_2$ could satisfy both equations for certain values of $a$, $a'$, $e_2$, and $e_2'$, where $X_1\ci e_2$ and $X_2\ci e_2'$. LiNGAM claims that this inconvenience occurs if and only if $X_1$ and $X_2$ are Gaussian. In other words, we can identify (\ref{eq11}) and (\ref{eq22}) if and only if at least one of $X_1$ and $X_2$ are non-Gaussian.

For the sufficient part, we show that if variables are both Gaussian, then causal order is unidentifiable. {Suppose} $X_1,X_2$ both are normally distributed, and the model (\ref{eq11}) with $X_1\ci e_2$ is true for certain $a$ and $\epsilon$.
Let 
$\sigma_1^2,\sigma_2^2$ be the variances of $e_1$ and $e_2$. Then, from ${\mathbb E}[e_1e_2]=0$, we have 
\begin{equation}\label{eq33}
e_1'=ae_1+e_2
\end{equation}
\begin{equation}\label{eq44}
e_2'=e_1-a'e_1'=e_1-a'(ae_1+e_2)=(1-a'a)e_1-a'e_2\ ,
\end{equation}
and
${\mathbb E}[e_1'e_2']=(1-a'a)\sigma^2_1-a'\sigma_2^2\ ,$
which means that choosing
\begin{equation}\label{eq55}
a'=\frac{a\sigma_1^2}{a^2\sigma_1^2+\sigma_2^2}
\end{equation}
will make the ${\mathbb E}[e_1'e_2']=0$ too. 
We call $W$ and $Z$ jointly Gaussian if the two random variables can be represented as
$
\left[
\begin{array}{c}
Z\\W
\end{array}
\right]
=A
\left[
\begin{array}{c}
U\\V
\end{array}
\right]
$
where $A\in {\mathbb R}^{2\times 2}$ and $U,V$ are independent Gaussian.

The well-known property states that independence is equivalent to zero correlation for jointly Gaussian variables\footnote{Suppose $Z$ and $W$ be binary taking $\pm 1$ equiprobably and zero-mean Gaussian. Then, 
$ZW$ and $Z$ are not jointly Gaussian.
Even though ${\mathbb E}[ZW\cdot Z]={\mathbb E}[W]\cdot{\mathbb E}[Z^2]=0$
but they are not independent.}. By checking $e_1'$ and $e_2'$ belonging to joint Gaussian distribution, we can conclude that $e_1'$ is independent of $e_2'$. Consequently,(\ref{eq22}) holds with $X_2\ci \epsilon'$ for the corresponding $a',\epsilon'$.

For the necessary part, assume that $X\ci \epsilon$ for (\ref{eq11}) and $Y\ci \epsilon'$ for (\ref{eq22}) both hold simultaneously for certain $a,a',\epsilon,\epsilon'$, where $a'$ satisfies (\ref{eq55}).
Therefore, this means that $a,a'\neq 0$ due to (\ref{eq23}) and (\ref{eq55}).
Now note the statement as follows:
\begin{proposition}[\cite{S,D}]\label{p1}\rm
Let $m\geq 2$ and independent random variables $\xi_{1},\ldots,\xi_m\in{\mathbb R}$.
Let two linear form $L_1=\sum_{i=1}^{m} \alpha_{i} \xi_{i}$ and $L_2=\sum_{i=1}^{m}\beta_{i} \xi_{i}$, if 
$L_1\ci L_2$,
for $\alpha_1,\ldots,\alpha_m$, $\beta_1,\ldots,\beta_m \in {\mathbb R}$.
Then the random variable $\xi_{i}$ such that $\alpha_{i} \beta_{i} \neq 0$ belongs to Gaussian for $i=1,\ldots,m$.
\end{proposition}

Following (\ref{eq33})(\ref{eq44})(\ref{eq55}) and the Proposition \ref{p1}, then 
\begin{equation*}\label{eq41}
\begin{aligned}
    (e_1,e_2,a,1,1-aa',-a')    &=
(\xi_1,\xi_2,\alpha_1,\alpha_2,\beta_1,\beta_2)
\\
&=(X,\epsilon,a,1,\frac{\sigma_2^2}{a^2\sigma_1^2+\sigma_2^2},-\frac{a\sigma_1^2}{a^2\sigma_1^2+\sigma_2^2})
\end{aligned}
\end{equation*}
By combining (\ref{eq23}), $X_1,\epsilon$ belong to Gaussian, then $X_2$ is also Gaussian-distributed.
\begin{proposition}[\cite{Direct}]\label{P2}\rm
Assuming (\ref{eq23}), we can identify the causal order using LiNGAM if at least one of two random variables belongs to non-Gaussian.
\end{proposition}

We can also identify the causal orders among multiple random variables. Suppose there are three linearly related random variables $X_1, X_2, X_3$ with zero means. 
Then, six potential causal orders exist, for instance, $X_2\rightarrow X_1\rightarrow X_3$, and $X_3\rightarrow X_2\rightarrow X_1$. First, we determine the top of them. 
Assuming $X_1$ is independent of $\{X_2-aX_1,X_3-a'X_1\}$ for $a,a'\in\mathbb{R}$, which means $X_1$ is the top variable. 
Furthermore, suppose that $X_2-aX_1$ is independent of $X_3-a'X_1-a''(X_2-aX_1)$ for some $a''\in {\mathbb R}$, then regarding the $X_2$ as the middle and $X_3$ as the bottom. We obtain the causal order $X_1\rightarrow X_2\rightarrow X_3$.
Following the steps, we can identify the causal order for $X_1, X_2, X_3$. 
Furthermore, we can estimate the causal order for an arbitrary number of random variables like 
$$X_i=\sum_{j=1}^{i-1}b_{i,j}X_j+e_i$$
where $b_{i,j}\in {\mathbb R}$ and noise $e_i$ is non-Gaussian for $p$ random variables $X_1,\ldots,X_p$.

\subsection{Hilbert Spaces}

A Banach space is a complete normed vector space where completeness ensures that all Cauchy sequences converge within the space. It combines linearity, completeness, and the norm to provide a framework for studying mathematical structures and functions. More precisely, in our context, we consider the set of functions as a Hilbert space, denoted by $\mathscr{H}$. A Hilbert space is a Banach space equipped with an inner product that induces the norm, ensuring completeness.

We define a linear operator $T_{21}: \mathscr{H}_1\rightarrow \mathscr{H}_2$ over ${\mathbb R}$ as a mapping that satisfies the linearity property: $T_{21}(\alpha f+ \beta g)=\alpha T_{21}f+\beta T_{21}g$ for $f,g\in \mathscr{H}_1$ and $\alpha,\beta\in {\mathbb R}$. Furthermore, $T_{21}$ is said to be bounded if there exists a positive constant $C$ such that $\|T_{21}f\|_{2}\leq C\|f\|_1$ holds for all $f \in \mathscr{H}_1$. Here, $\|\cdot\|_1$, $\|\cdot\|_2$ denote the norms within $\mathscr{H}_1$, $\mathscr{H}_2$, respectively.

For any bounded operator $T_{21}: \mathscr{H}_1\rightarrow \mathscr{H}_2$, there exists its adjoint operator or dual operator, a unique bounded linear operator $T_{21}^*: \mathscr{H}_2\rightarrow \mathscr{H}_1$ such that the following equality holds: $\langle T_{21}f_1,f_2\rangle_2=\langle f_1,T_{21}^*f_2\rangle_1$ for $f_1\in \mathscr{H}_1$ and $f_2\in \mathscr{H}_2$. The operator $T_{21}^*$ is the adjoint operator of $T_{21}$. If $T_{21}=T_{21}^*$, we say that $T_{21}$ is self-adjoint. Moreover, if the dimension of $\mathscr{H}$ is finite, the self-adjoint operator $T_{21}$ is symmetric.

\subsection{Random functions}

Functional data analysis involves considering each individual element of data as a random function. These functions are defined over a continuous physical continuum, which is typically time but can also be spatial location, wavelength, probability, or other dimensions. Functionally, these data are infinite-dimensional. Random functions can be interpreted as random elements that take values in a Hilbert space or as stochastic processes. The former approach provides mathematical convenience, while the latter is more suitable for practical applications. These two perspectives align when the random functions are continuous and satisfy a mean-squared continuity condition.

Formally speaking, if a mapping $X: \Omega \rightarrow {\mathbb R}$ is measurable from a probability space 
$(\Omega,{\cal F},\mu)$ to $({\mathbb{R}},{\mathcal {B}}({\mathbb{R}}))$, then it is a random variable:
$$B\in {\mathcal {B}}({\mathbb {R}}) \Longrightarrow   \{\omega\in \Omega|X(\omega)\in B \} \in {\cal F}\ ,$$
with the Borel sets $\mathcal{B}({\mathbb {R}})$.
Similarly, if $\chi: \Omega \rightarrow \mathscr{H}$ is measurable 
 from $(\Omega,{\cal F},\mu)$ to $(\mathscr{H},{\cal B}(\mathscr{H}))$, then it is a random function (or random element) in a Hilbert space $\mathscr{H}$:
 $$B\in {\mathcal {B}}({\mathbb {\mathscr{H}}}) \Longrightarrow   \{\omega\in \Omega|X(\omega)\in B \} \in {\cal F}\ ,$$
with the Borel sets $\cal B(\mathscr{H})$ w.r.t. the norm of $\mathscr{H}$.
Let $E$ be one set, we suppose that every entry $f$ of $\mathscr{H}$ is a function 
$f: E\ni x \mapsto f(x) \in \mathbb R$.

The mean of the random function $\chi$ is defined using the Bochner integral\footnote{See the definition of the Bochner integral in \cite{fdabook}.} as $\int_\Omega \chi d\mu$, under the condition that the expectation of $\|\chi\|$ is bounded.
Moreover, if the means of $\chi_1,\chi_2$ in $\mathscr{H}$ are $m$,  
we give the definition of the covariance operator $\mathscr{K}: \mathscr{H}\rightarrow \mathscr{H}$ of random functions $\chi_1,\chi_2$ when $\mathscr{H}:=\mathscr{H}_1=\mathscr{H}_2$: 
$$\langle \mathscr{K} g_1,g_2\rangle=\langle \int_{\Omega}\langle \chi_1-m,g_1\rangle(\chi_2-m)\rangle d\mu,g_2\rangle=
\int_{\Omega} \langle \chi_1-m,g_1\rangle\langle \chi_2-m, g_2\rangle d\mu\ ,
$$ 
for $g_1, g_2\in \mathscr{H}$.
By using orthonormal bases $\{e_i\}$ in $\mathscr{H}$, we can compute the covariance values $\langle \mathscr{K} e_i,e_j\rangle$ for all pairs of indices $i$ and $j$.
Generally, if $\chi_1\ci\chi_2$, 
then we get $\langle \mathscr{K}\ g_1,g_2\rangle=0$ for $g_1,g_2\in \mathscr{H}$.

In the context where each element in $\mathscr{H}$ is a mapping from $E$ to $\mathbb{R}$, a random function $\chi: \Omega\rightarrow \mathscr{H}$ takes values $\chi(\omega, x)\in \mathbb{R}$ for each $\omega\in \Omega$ and $x\in E$. Furthermore, if we fix $\omega\in \Omega$, $\chi(\omega, \cdot)$ represents a random function from $E$ to $\mathbb{R}$.
Henceforth, we adopt the notation $\chi(\cdot)$ to represent the random function $\chi(\omega, \cdot)$. This convention is analogous to the simplification employed for random variables, where $X(\omega)$ is denoted as $X$.
Note that a mean $m$ random function $\chi$ is referred as a Gaussian process if for any $n\geq 1$, the random vector $[\chi(x_1), \ldots, \chi(x_n)]$ of length $n$ follows a Gaussian distribution with mean $[m(x_1), \ldots, m(x_n)]$, $x_1, \ldots, x_n \in E$.

When the Hilbert space $\mathscr{H}$ has a finite dimension $d$, the covariance operator can be represented by a covariance matrix, denoted as $\Sigma \in \mathbb{R}^{d \times d}$. This matrix is positive definite. Consequently, we can define the eigenvalues $\{\lambda_i\}$ and eigenvectors $\{\phi_i\}$ of $\Sigma$. Each vector in $\mathscr{H}$ can be expressed as a linear combination of the eigenvectors, specifically as $\sum_{i=1}^d \langle X, \phi_i \rangle \phi_i$. Moreover, for $\langle X, \phi_i \rangle$, the variance is given by $\lambda_i$. Then, for random function $\chi$,
if $\mathscr{H}$ is an infinite-dimensional function space, 
\begin{proposition}[\cite{fdabook}]\label{p3}
Let $\{\lambda_i\}$ and $\{\phi_i\}$ denote the eigenvalues and eigenfunctions obtained from the eigenvalue problem $\mathscr{K} \phi_i = \lambda_i \phi_i$, $i=1,2,\ldots$. With probability one, $\chi$ can be represented as:
$$
\chi = \sum_{i=1}^\infty \langle \chi, \phi_i \rangle_\mathscr{H} \phi_i,
$$
where $\langle \chi, \phi_i \rangle_\mathscr{H}$ denotes the inner product between $\chi$ and $\phi_i$ in $\mathscr{H}$. Additionally, mean of $\chi$ is zero, and for $\langle \chi, \phi_i \rangle_\mathscr{H}$, the variance is equal to $\lambda_i$.
\end{proposition}

It is important to note the close relationship between stochastic processes and random functions. A set of random variables $\{X(t)\}_{t\in E}$ can be considered a stochastic process if the function $X: \Omega \times E \rightarrow {\mathbb R}$ is measurable with respect to the probability space $(\Omega, \mathcal{F}, \mu)$ and the measurable space $({\mathbb R}, \mathcal{B}({\mathbb R}))$ for each $t\in E$. It is worth mentioning that certain stochastic processes can also be regarded as random functions \citep{fdabook}.
\section{Extension to Functional Data}\label{s3}
In this section, we generalize the concept of LiNGAM from random variables to encompass both random vectors and random functions.

Previous works have extended the D-S to encompass various scenarios. These extensions include incorporating random vectors \citep{MDS} and random functions in a Banach space \citep{DSB} as substitutes for random variables.
\subsection{LiNGAM for Random Vectors}
 As shown by \cite{Direct}, the identifiability of non-Gaussian random variables is outlined in Proposition \ref{P2}. However, this proposition does not extend to the case of random vectors or random functions. This section provides proof of identifiability for non-Gaussian random vectors.

 \begin{proposition}[Multivariate Darmois-Skitovich \citep{MDS}]\label{p5}	
  Let $ {L}_{1}=   \sum_{i=1}^{m} {A}_{i} {\xi}^{i} $ and $ {L}_{2}=\sum_{i=1}^{m} {B}_{i} {\xi}^{i} $ with mutually independent $ k $-dimensional random vectors $ {\xi}^{i} $ and invertible matrices {$ {A}_{i}, {B}_{i}$} for $ i=1, \ldots, m $. If $ {L}_{1} $ and $ {L}_{2} $ are mutually independent, then all $ {\xi}^{i} $ are Gaussian.
	\end{proposition}
	Now we consider the identifiability of the following model when $x,y\in\mathbb{R}^m$ and invertible matrix $A\in\mathbb{R}^{m\times m}$, $e_1\ci e_2$ and zero means,
	\begin{equation}
		\begin{aligned}
			& x = \epsilon_1,      \ \ \ \  & y &= \epsilon_1', \\
			& y = Ax + \epsilon_2, \ \ \ \  & x &= A'y + \epsilon_2',\\
			&\epsilon_1' = A\epsilon_1+\epsilon_2,\ \ 
			&\epsilon_2'& = (I-A'A)\epsilon_1-A'\epsilon_2.
		\end{aligned}
	\end{equation}
We assume 
\begin{equation}\label{eq233}
A\ {\rm or}\ A'\ {\rm is\ invertible}.
\end{equation}
Then, we have the following theorem.
 \begin{theorem}
\label{t1}\rm
Assuming (\ref{eq233}), which extends (\ref{eq23}), we can identify the causal order between random vectors {$X_1, X_2:\Omega\rightarrow\mathbb{R}^m$ of dimension $m\in [1,\infty)$} if and only if at least one of them is non-Gaussian.
\end{theorem}
	\begin{proof}
		Since $\epsilon_{1} \perp\!\!\!\perp \epsilon_{2}, E \epsilon_{1}^{\prime}=E \epsilon_{2}^{\prime}=0 $, and they are {Gaussian random vectors with covariance matrix $\Sigma_1,\Sigma_2$, respectively}. Then the correlated coefficient $\rho=0 \Longleftrightarrow Cov(\epsilon_{1}',\epsilon_{2}')=A\Sigma_1\left(I-A^TA^{\prime T}\right) -\Sigma_{2}A'^T=0\Longleftrightarrow\epsilon_1'\perp\!\!\!\perp \epsilon_2'$, 
		that is, when $A^{\prime}=\Sigma_1A^T(A\Sigma_1A^T+\Sigma_2)^{-1}$, the causal relation between $x,y$ is unable to be identified. This also satisfies the condition of $\epsilon_1'\perp\!\!\!\perp \epsilon_2'$ is that they follow the
		Gaussian distribution from the Proposition \ref{p5}. 
	\end{proof}

\subsection{LiNGAM for Random Functions}
In this subsection, we present results that demonstrate identifiability can be achieved in non-Gaussian scenarios in infinite-dimensional Hilbert spaces as Theorem~\ref{t6}. In extending our approach to multivariate scenarios, we adopted methodologies from Lemmas 1 and 2 of DirectLiNGAM (refer to \citep{Direct}). This involved identifying the exogenous function (see Appendix) and using residuals for causal ordering, paralleling the process in Direct-LiNGAM. Owing to the procedural similarities with multivariate functions, we omitted a detailed proof in our main text, choosing to apply these established principles to our context. We have included the preliminary proof in Section~\ref{AB} for clarity.

Let $\mathscr{H}_1, \mathscr{H}_2$ be Hilbert spaces. Assume that there are two causal models
for $f_1\in \mathscr{H}_1$ and $f_2\in \mathscr{H}_2$,
\begin{equation}\label{eq19}
	\begin{aligned}
		& f_1 = h_1, \quad  & f_2 &= T_{21}f_1+h_2, \\
		& f_2=h_1', \quad & f_1&=T_{12}f_2 + h_2'.
	\end{aligned}
\end{equation}
where random functions $\{h_1,h_2'\}\in\mathscr{H}_1$ and $\{h_1',h_2\}\in\mathscr{H}_2$. We also assume the covariance operator $\mathscr{K}_{11}$ of $h_1$, $\mathscr{K}_{22}$ of $h_2$ have positive eigenvalues ($>0$). The $T_{12}: \mathscr{H}_2\rightarrow \mathscr{H}_1$, $T_{21}: \mathscr{H}_1\rightarrow \mathscr{H}_2$
are linear bounded operators between $\mathscr{H}_1,\mathscr{H}_2$,  and we identify the order by examining whether $h_2\ci f_1$ or $h_1\ci f_2$.

A bounded linear operator $T: \mathscr{H}_1\rightarrow \mathscr{H}_2$ is considered continuous if the set $\{T(f)|f\in U\}\subseteq \mathscr{H}_2$ is open for any subset $U\subseteq \mathscr{H}_1$. Similarly, the inverse image $U$ is also open. Furthermore, an operator $T: \mathscr{H}_1\rightarrow \mathscr{H}_2$ is said to be invertible if it is both one-to-one (injective) and onto (surjective).

Let's confirm the statements before proceeding with our discussion:
\begin{itemize}
\item Proposition \ref{p6}: There is an equivalence between independence and non-correlation for jointly Gaussian random functions. In other words, if $\chi_1$ and $\chi_2$ are jointly Gaussian random functions, they are independent if and only if they are uncorrelated.

\item Proposition \ref{p2}: The Darmois-Skitovich (D-S) theorem can be extended to random functions in Banach spaces.

\end{itemize}

The following Proposition \ref{p6} establishes the equivalence between independence and non-correlation for random functions in Banach spaces, which also includes Hilbert spaces as a special case. 
\begin{proposition}[\cite{IC}]\label{p6}
Suppose $\chi,\chi'$ are joint Gaussian in Banach spaces. Then, $\chi\ci\chi'$ if and only if they are uncorrelated. 
\end{proposition}
\begin{proposition}[Darmois-Skitovich in Banach Space\citep{DSB}]\label{p2}
Suppose that $n\geq 2$, and random functions $\xi_{1},\ldots,\xi_n$ are in a Banach space. Let $L_1=\sum_{i=1}^{m} A_{i} \xi_{i}$, 
$L_2=\sum_{i=1}^{m}B_{i} \xi_{i}$ with some continuous linear bounded operators $A_1,\ldots,A_m$, and $B_1,\ldots,B_m$. If $L_1\ci L_2$, then $\xi_{i}$ is a Gaussian process for $i=1,\ldots,m$ with invertible $A_{i}, B_i$.
\end{proposition}



\begin{theorem}[\textbf{Causal 
Identifiability}]\label{t6}
If either $T_{12}$ or $T_{21}$ is invertible, the causal order between random functions in infinite-dimensional Hilbert spaces can be identified if and only if at least one of them is a non-Gaussian process.
\end{theorem}
\begin{proof}
For the sufficiency, from (\ref{eq19}), we first assume the $f_1=h_1$, $f_2=T_{21}f_1+h_2$, and represent the noise functions $h_1',h_2'$ with $h_1,h_2$:
	\begin{equation}
		\begin{aligned}
			h_1'&=f_2=T_{21}h_1+h_2\\
			h_2'&=f_1-T_{12}f_2=h_1-T_{12}(T_{21}h_1+h_2)=(I-T_{12}T_{21})h_1-T_{12}h_2\ .
		\end{aligned}
	\end{equation}
Because $h_1',h_2'$ are formed as the linear combinations of two independent Gaussian random functions $h_1,h_2$, we can conclude that $h_1'$ and $h_2'$ are jointly Gaussian \citep{IC}. 
Then from Proposition \ref{p6}, the zero-correlation implies independence. 
Since $h_1\ci h_2$ and $h_1\in \mathscr{H}_1$, $h_2\in \mathscr{H}_2$,
the cross-covariance operator $\mathscr{K}_{12}$ is zero: $$\langle\mathscr{K}_{12}g_1, g_2\rangle_{H_2}=\int_\Omega \langle h_1,g_1\rangle_{\mathscr{H}_1}\langle h_2,g_2\rangle_{\mathscr{H}_2}=0$$ for any $g_1\in \mathscr{H}_1$,$g_2\in \mathscr{H}_2$. Then, 
the cross-covariance operator $\mathscr{K}_{12}'$ between $h_1'$ and $h_2'$ is
	\begin{equation}\label{eq21}
		\begin{aligned}
			&\langle \mathscr{K}'_{12}g_1,g_2\rangle_{\mathscr{H}_2}=\int_{\Omega}
			\langle (I-T_{12}T_{21})h_1-T_{12}h_2, g_1 \rangle_{\mathscr{H}_1} \langle T_{21}h_1+h_2, g_2\rangle_{\mathscr{H}_2}
			d\mu\\
			&=
			\int_{\Omega}\langle (I-T_{12}T_{21})h_1, g_1 \rangle_{\mathscr{H}_1} \langle T_{21}h_1, g_2\rangle_{\mathscr{H}_2}
			d\mu+
			\int_{\Omega}\langle -T_{12}h_2, g_1 \rangle_{\mathscr{H}_1} \langle h_2, g_2\rangle_{\mathscr{H}_2}
			d\mu
			\\
			&=
			\int_{\Omega}\langle h_1,  (I-T_{12}T_{21})^*g_1 \rangle_{\mathscr{H}_1} \langle h_1, T_{21}^* g_2\rangle_{\mathscr{H}_1}
			d\mu-
			\int_{\Omega}\langle h_2,  T_{12}^* g_1 \rangle_{\mathscr{H}_2} \langle h_2, g_2\rangle_{\mathscr{H}_2}d\mu\\
			&=
			\langle \mathscr{K}_{11}(I-T_{12}T_{21})^*g_1,T_{21}^*g_2\rangle_{\mathscr{H}_1}-\langle\mathscr{K}_{22}T_{12}^*g_1,g_2\rangle_{\mathscr{H}_2}\\
			&=\langle T_{21}\mathscr{K}_{11}(I-T_{21}^*T_{12}^*)g_1,g_2\rangle_{\mathscr{H}_2}-\langle\mathscr{K}_{22}T_{12}^*g_1,g_2\rangle_{\mathscr{H}_2}
		\end{aligned}
	\end{equation}
for any $g_1\in \mathscr{H}_1$,$g_2\in \mathscr{H}_2$, where 
 $\mathscr{K}_{11}, \mathscr{K}_{22}$ are the covariance operators of $h_1,h_2$, respectively. We assume that $\mathscr{K}_{11},\mathscr{K}_{22}$
 are not zero.
If $\mathscr{K}_{12}'=0$, then we require
\begin{equation}\label{eq31}
			\mathscr{K}_{11}T_{21}^*=
			T_{12}\{T_{21}\mathscr{K}_{11}T_{21}^*+\mathscr{K}_{22}\}\ .
\end{equation}
We have 
\begin{equation*}
\begin{aligned}
    (\ref{eq21})=0&\Leftrightarrow
T_{21}\mathscr{K}_{11}(I-T_{21}^*T_{12}^*)=\mathscr{K}_{22}T_{12}^*
\\
&\Leftrightarrow T_{21}\mathscr{K}_{11}
=(T_{21}\mathscr{K}_{11}T_{21}^*+\mathscr{K}_{22})T_{12}^*
\Leftrightarrow (\ref{eq31})
\end{aligned}
\end{equation*}
However, the covariance operator $K_{11}$ and $K_{22}$ are not invertible because of they are compact operator:
\begin{itemize}
    \item A covariance operator is trace-class operator (Theorem 7.2.5 in \cite{fdabook});
    \item A trace-class operator is Hibert-Schmidt operator (Theorem 4.5.2 in \cite{fdabook});
    \item An Hilbert–Schmidt operator is compact (Theorem 4.4.3 in \cite{fdabook});
    \item A compact operator is not invertible (Theorem 4.1.4 in \cite{fdabook}).
\end{itemize}
Then we know covariance operators are not invertible. But here, we need to notice that we can always define a Moore-Penrose inverse to make the equation (\ref{eq31}) hold if
\begin{equation}\label{c1}
\operatorname{Im}\left(\mathscr{K}_{11}T_{21}^*\right) \subseteq \operatorname{Im}\left(\{T_{21}\mathscr{K}_{11}T_{21}^*+\mathscr{K}_{22}\}\right)
\end{equation}
 and the following is bounded \citep{Libing}:
 \begin{equation}\label{c2}
 \{T_{21}\mathscr{K}_{11}T_{21}^*+\mathscr{K}_{22}\}^{\dagger} \mathscr{K}_{11}T_{21}^*\ .
 \end{equation} 
 Then the problem becomes determining the Images and boundness.

Next we prove $\operatorname{Im}(A) \subseteq \operatorname{Im}(A+B)$. Note that if $A$  is positive semidefinite and $ \langle A u, u\rangle=0$, then $A u=0$. To see why, let
 $v_{1}, \ldots, v_{n}$ be an orthonormal basis of eigenvectors of $A$ (so $A$ $v_{i}=\lambda_{i} v_{i}$) and write  $u=\sum_{i=1}^{n}\left\langle u, v_{i}\right\rangle v_{i} $. Then
$$
\langle A u, u\rangle=\sum_{i=1}^{n}\left\langle u, v_{i}\right\rangle^{2} \lambda_{i}=0
$$
together with $\lambda_{i} \geq 0$  implies that $\left\langle u, v_{i}\right\rangle=0 $ if $\lambda_{i}>0$  so  $u \in \operatorname{ker}(A) $.
To prove that $\operatorname{Im}(A) \subseteq \operatorname{Im}(A+B)$, it is enough to prove that
\begin{equation}
    \begin{array}{c}
\operatorname{ker}(A+B)=\operatorname{Im}(A+B)^{\perp} \subseteq \operatorname{Im}(A)^{\perp}
=\operatorname{ker}(A)
\end{array}
\end{equation}
let $ u \in \operatorname{ker}(A+B) $. Then
$
0=\langle(A+B) u, u\rangle=\langle A u, u\rangle+\langle B u, u\rangle
$
which implies that $\langle A u, u\rangle=0$, so $ u \in \operatorname{ker}(A)$. Then (\ref{c1}) satisfys. Now we consider the boundness. As we know, the eigenvalue of $A+B$ (positive semidefinite) is bigger than $A$ or $B$, which means the inverse eigenvalue of $A+B$ will be smaller than the inverse eigenvalue of $A$ or $B$. Moreover, the smallest eigenvalue of the covariance operator tends to $0$, then $(A+B)^\dagger A$ is bounded. Then we say the equation (\ref{eq31}) holds. We can check more details in the Appendix.



Conversely, we first let $h_1\ci h_2$ and $h_1'\ci h_2'$ in (\ref{eq19}) hold true simultaneously for some $T_{12}$, $T_{21}$,
and we want to prove that $h_1,h_2,h_1',h_2'$
belong to Gaussian under (\ref{eq31}). 
Note that a Hilbert space is a special case of Banach space. Then we use the Proposition \ref{p2}.
We assume that $T_{12}$ is invertible without losing generality.
Next we show that the eigenvalue of $T_{12}T_{21}$  is less than 1, which means that 
$I-T_{12}T_{21}$
is invertible (see Theorem 3.5.5 in  \cite{fdabook}).
To achieve this, we multiply (\ref{eq31}) by $T_{21}$ from the left-hand side, then we obtain
$$T_{21}\mathscr{K}_{11}T_{21}^*=
			T_{21}T_{12}\{T_{21}\mathscr{K}_{11}T_{21}^*+\mathscr{K}_{22}\}\ ,$$
which means that the eigenvalue of $T_{21}T_{12}$ is less than 1. 
Noting that $T_{21}T_{12}$ and $T_{12}T_{21}$ share the eigenvalues:
$$T_{21}T_{12}u=\lambda u 
\Longrightarrow T_{12}T_{21}T_{12}u=\lambda T_{12}u
\Longrightarrow T_{12}T_{21}v=\lambda v
$$
for $\lambda\not=0$, $u\in \mathscr{H}_2$, and $v:=T_{12}u\in \mathscr{H}_1$, we have proved 
that the eigenvalue of $T_{12}T_{21}$ is less than 1. Then, as we did in (\ref{eq41}), we correspond 
\begin{eqnarray*}
&&\left(h_1.h_2,T_{21},I,
I-T_{21}T_{12},-T_{12}\right)
=\left(\xi_1,\xi_2,A_1,A_2,B_1,B_2)
\right)\ ,
\end{eqnarray*}
where $A_1,A_2,B_1,B_2$ are invertible.

\end{proof}

{
\subsection{Causal Inference in Multivariate Scenarios}\label{AB}
In the context of multivariate cases, we introduce two lemmas following \cite{Direct}:\\
1. Lemma~\ref{lm1} identifies the exogenous function.\\
2. Lemma~\ref{lm2} establishes the causal order among residuals.\\
By analyzing residuals, we can determine the causal order of random functions. This is achieved after identifying an exogenous function, which, under the assumption of no latent confounders, corresponds to an independent external influence. The independence of these residuals is assessed through a series of pairwise regressions.
\begin{lemma}\label{lm1}
    For multivariate case, a random function $f_j$ is exogenous if and only if $f_j$ is independent of its residuals $h_i^{(j)}=f_i-{T}_{ij}f_j$ for all $i\neq j$.
\end{lemma}
\begin{proof}
    For the sufficiency, if $f_j\ci h_i^{(j)}$, assume $f_j$ is not exogenous, then $f_j=\sum_{k\in P_j}{T}_{jk}f_k+h_j=\sum_{k\in P_j}{T}_{jk}\sum_{l\neq j}T_{kl}h_l+h_j$, where $P_j$ means parents of $f_j$. Then $h_i^{(j)}=(I-{T}_{ij}{T}_{ji})f_i-{T}_{ij}\sum_{k\in P_j,k\neq i}{T}_{jk}f_k-{T}_{ij}h_j=(I-{T}_{ij}{T}_{ji})\sum_{q\neq j}T_{iq}h_q-{T}_{ij}\sum_{k\in P_j,k\neq i}{T}_{jk}\sum_{l\neq j}T_{kl}h_l-{T}_{ij}h_j$. The two formulas are composed of linear combinations of
external influences other than $h_j$, from Prop.~\ref{p2}, all the functions are non-Gaussian, then $h_i^{(j)}\nni f_j$, then it contradicts. Therefore, $f_j$ should be exogenous; For the necessity, if $f_j$ is exogenous, $f_j=h_j$, $f_i=T_{ij}f_j+h_i$ with $h_i\ci f_j,h_i=\sum_{k\neq j}T_{ik}h_k$, we know the residual error $h_i^{(j)}=h_i$. Then, we know $f_j\ci h_i^{(j)}$ from the independence of noise functions. So far, the lemma has been proven.
\end{proof}
\begin{lemma}\label{lm2}
    Let $k_{r^{(j)}}(i)$ is the causal order of $r_i^{(j)}$, $k(i)$ denotes a causal
order of $f_i$. Then, the same ordering of the residuals $r_i=h_i^{(1)}=f_i-T_{i1}f_1,i=1\dots,p-1$ is a causal ordering for the original observed
functions as well: $k_{r^{(j)}}(l)<k_{r^{(j)}}(m)\Longleftrightarrow k(l)<k(m)$.
\end{lemma}
\begin{proof}
    When we determine the exogenous function $f_1$, we need to estimate the $p-1$ residuals of $f_1$: $r_i=h_i^{(1)}=f_i-T_{i1}f_1=\sum_{j\neq 1}T_{ij}f_j+T_{i1}f_1-T_{i1}f_1=\sum_{j\neq 1}T_{ij}f_j,i=1,\dots,p-1$, which is $r_i=\sum_{j\neq 1}T_{ij}\sum_{k\neq j}T_{jk}h_k$. For the residual of $r_2=f_2-T_{21}f_1=h_2$ (second function) is $r_i^{(j)}=r_i-T'_{ij}r_j=\sum_{j\neq 1}T_{ij}\sum_{k\neq j,2}T_{jk}h_k+T_{i2}'h_2-T_{i2}'h_2=\sum_{j\neq 1}T_{ij}\sum_{k\neq j,2}T_{jk}h_k$ for all $i\neq j.$ From the independence assumption of noise functions, we know $r_2\ci r_i^{(2)}$. Then we know the causal relationships of residuals $r_i,i=1,\dots,p-1$ are the same as $f_i,i=1,\dots,p-1$ with the $T'_{ij}$ because what we need to do is to test the independence between $r_i$ and its residual $r_i^{(j)}$. 
\end{proof}}
Extending the notion, we can determine the order among any number of random functions such as 
$f_i=\sum_{j=1}^{i-1}T_{i,j}f_j+h_i$
with non-Gaussian $h_i$ and bounded linear operators $T_{i,j}; H_{j}\rightarrow H_i$ for $p$ random functions $f_1\in H_1,\ldots,f_p\in H_p$.

\section{The Procedure}\label{s4}
Consider one model from (\ref{eq19}):
\begin{equation}\label{eq12}
f_2=T_{21}f_1+h_2\ .
\end{equation}
Then let's notice the statement as follows: 
\begin{proposition}[\cite{fdabook}]\label{prop6}
Let $T: \mathscr{H}_1\rightarrow \mathscr{H}_2$ be a compact\footnote{We define a bounded linear operator $T: \mathscr{H}_1\rightarrow \mathscr{H}_2$ to be compact if, for any bounded infinite sequence $\{f_n\}$ in $\mathscr{H}_1$, the sequence $\{Tf_n\}$ has a convergent subsequence in $\mathscr{H}_2$.} bounded linear operator, $\{\lambda_j\}$ be the eigenvalues, and $\{e_{1,j}\}$ and $\{e_{2,j}\}$ be the sequences with orthonormal eigenvectors of 
$T^*T$ and $TT^*$, respectively.
Then
$$Tf=\sum_{i=1}^\infty \lambda_i\langle f,e_{1,i}\rangle_{\mathscr{H}_1}e_{2,i}$$
with $f\in \mathscr{H}_1$.
\end{proposition}

Following the notation in Proposition \ref{prop6}, we write the three terms $T_{21}f_1=\sum_{i=1}^\infty\lambda_i f_{1,i}e_{1,i}$, $f_2=\sum_{i=1}^\infty f_{2,i}e_{2,i}$, $h_2=\sum_{i=1}^\infty h_{2,i}e_{2,i}$. Then, 
(\ref{eq12}) becomes:
\begin{theorem}\label{t7}
Suppose that $T_{21}: \mathscr{H}_1\rightarrow \mathscr{H}_2$ is compact.
If we regard the bases of $\mathscr{H}_1$ and $\mathscr{H}_2$ as $\{e_{1,i}\}$ and $\{e_{2,i}\}$, respectively,
then
\begin{equation}\label{eq13}
f_{2,i}=\lambda_i f_{1,i}+h_{2,i}
\end{equation}
for $i=1,2,\ldots$, where $\lambda_1\geq \lambda_2\geq \cdots$.
\end{theorem}


To ensure convergence of the eigenvalue sequence $\{\lambda_i\}$, we suppose that the operator $T_{21}$ is compact. Without compactness, the $\{\lambda_i\}$ would not be convergent. Practically, we approximate the infinite-dimensional random functions $f_1\in \mathscr{H}_1$, $f_2,h_2\in \mathscr{H}_2$ by finite length $M$ random vectors. We select the bases $\{e_{1,i}\}_{i=1}^M$ and $\{e_{2,i}\}_{i=1}^M$ to minimize the approximation error.

FPCA offers a more effective fit than PCA for raw data dimensionality reduction, particularly with time-series data like fMRI and EEG, where dimensions vary with sampling frequency (e.g., 100Hz vs. 1Hz). As frequency increases, dimensions approach infinity. FPCA overcomes this by approximating infinite dimensions through orthogonal bases, preserving maximal original data information and capturing latent details beyond traditional sampling. Figures~\ref{E1} and \ref{E2} demonstrate FPCA's necessity for functional data.

The other merit of using the FPCA (functional principal component analysis) approach is 
its efficiency. We assume the following procedure:
first, we approximate the $W$ time points sampled from functions by the $L$ coefficients of the basis functions (B-spline). Then, we transform it by the $M$ coefficients of the basis functions
defined above. The time complexity is as follows. $M< L\ll W$ and 
the time complexity $C(M)$ of the proposed procedure is much less than $C(W)$.
For example, \cite{Direct} evaluated the complexity of their method as $C(W)=O(n(Wp)^3q^2+Wp)^4q^3)$, where $q$ ($\ll n$) is the maximal rank found by the low-rank decomposition used in the kernel-based independence measure, 
although the proposed procedure requires additional $O(nL^2+L^3)$ complexity for  the covariance matrix $O(nL^2)$ and eigenvalue decomposition $O(L^3)$. 

This paper primarily examines the summary causal relationships among random functions, focusing less on specific time points or partial windows in temporal data. There are three graphical representations of causal structures in temporal data, namely, the \textit{full-time causal graph}, the \textit{window causal graph}, and the \textit{summary causal graph} \citep{survey}. The \textit{full-time causal graph}, illustrated on the left in Fig.~\ref{E4}, depicts a complete dynamic system, representing all vertices including components \(f_1, \dots, f_p\) at each time point \(t\), connected through lag-specific directed links such as \(f_{i}^{t-k}\rightarrow f_{j}^{t}\). However, due to the challenges of capturing a single observation for each series at every time point, constructing a full-time causal graph can be complex.
To address this, the \textit{window causal graph} concept is introduced, which operates under the assumption of a time-homogeneous causal structure. This graph, shown in the middle of Fig.~\ref{E4}, works within a time window corresponding to the maximum lag in the full-time graph. On the other hand, the \textit{summary causal graph}, displayed on the right in Fig.~\ref{E4}, abstracts each time series component into a single node, illustrating inter-series causal relationships without specifying particular time lags. The complexity of this summary graph depends on the choice of multivariate dependence measure, such as mutual information or HSIC. The algorithmic complexity for generating this graph is similar to that of DirectLiNGAM. Fig.~\ref{E4} visually compares these different types of causal graphs for multivariate time series.


\begin{figure*}[t]
    \centering
\begin{tikzpicture}\label{MTS}

\node[draw, circle, inner sep=1pt, font=\scriptsize] (a) at (0,2.2) {$f_1^{t-2}$};
\node[circle, draw=black, inner sep=1pt, font=\scriptsize] (b) at (1.5,2.2) {$f_1^{t-1}$};
\node[circle, draw=black, inner sep=3.7pt, font=\scriptsize] (c) at (3.0,2.2) {$f_1^{t}$};
\node[circle, draw=black, inner sep=1pt, font=\scriptsize] (d) at (4.5,2.2) {$f_1^{t+1}$};
\node[draw, circle, inner sep=1pt, font=\scriptsize] (a1) at (0,1.1) {$f_2^{t-2}$};
\node[circle, draw=black, inner sep=1pt, font=\scriptsize] (b1) at (1.5,1.1) {$f_2^{t-1}$};
\node[circle, draw=black, inner sep=3.7pt, font=\scriptsize] (c1) at (3.0,1.1) {$f_2^{t}$};
\node[circle, draw=black, inner sep=1pt, font=\scriptsize] (d1) at (4.5,1.1) {$f_2^{t+1}$};
\node[draw, circle, inner sep=1pt, font=\scriptsize] (a2) at (0,0) {$f_3^{t-2}$};
\node[circle, draw=black, inner sep=1pt, font=\scriptsize, label=below:Full-time] (b2) at (1.5,0) {$f_3^{t-1}$};
\node[circle, draw=black, inner sep=3.7pt, font=\scriptsize] (c2) at (3.0,0) {$f_3^{t}$};
\node[circle, draw=black, inner sep=1pt, font=\scriptsize] (d2) at (4.5,0) {$f_3^{t+1}$};
\draw[->,thick] (a) -- (b);
\draw[->,thick] (b) -- (c);
\draw[->,thick] (c) -- (d);
\draw[->,thick] (a1) -- (b1);
\draw[->,thick] (b1) -- (c1);
\draw[->,thick] (c1) -- (d1);
\draw[->,thick] (a2) -- (b2);
\draw[->,thick] (b2) -- (c2);
\draw[->,thick] (c2) -- (d2);
\draw[->,thick] (a1) -- (b2);
\draw[->,thick] (b1) -- (c2);
\draw[->,thick] (c1) -- (d2);
\draw[->,thick] (a1) -- (a);
\draw[->,thick] (b1) -- (b);
\draw[->,thick] (c1) -- (c);
\draw[->,thick] (d1) -- (d);

\node[draw, circle, inner sep=1pt, font=\scriptsize] (a3) at (6.5,2.2) {$f_1^{t-1}$};
\node[circle, draw=black, inner sep=3pt, font=\scriptsize] (b3) at (8.0,2.2) {$f_1^{t}$};
\node[draw, circle, inner sep=1pt, font=\scriptsize] (a4) at (6.5,1.1) {$f_2^{t-1}$};
\node[circle, draw=black, inner sep=3pt, font=\scriptsize] (b4) at (8,1.1) {$f_2^{t}$};
\node[draw, circle, inner sep=1pt, font=\scriptsize, label=below: Window] (a5) at (6.5,0) {$f_3^{t-1}$};
\node[circle, draw=black, inner sep=3pt, font=\scriptsize] (b5) at (8,0) {$f_3^{t}$};
\draw[->,thick] (a3) -- (b3);
\draw[->,thick] (a4) -- (b4);
\draw[->,thick] (a5) -- (b5);
\draw[->,thick] (a4) -- (b5);
\draw[->,thick] (a4) -- (a3);
\draw[->,thick] (b4) -- (b3);

\node[draw, circle, inner sep=4pt, font=\scriptsize] (a3) at (10,2.2) {$f_1$};
\node[draw, circle, inner sep=4pt, font=\scriptsize] (a4) at (10,1.1) {$f_2$};
\node[draw, circle, inner sep=4pt, font=\scriptsize, label=below: Summary] (a5) at (10,0) {$f_3$};

\draw[->,thick] (a4) -- (a3);
\draw[->,thick] (a4) -- (a5);


\end{tikzpicture}
    \caption{\textbf{Illustration of Different Kinds of Multivariate Time Series Causal Graphs}. Left: Full-time; Middle: Window; Right: Summary (this paper).}
    \label{E4}
\end{figure*}
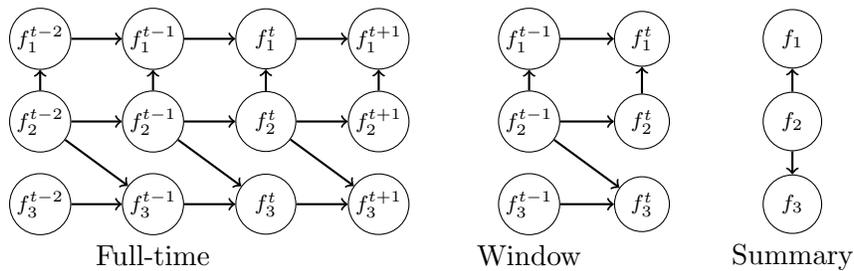

\subsection{Algorithm}
To show how to implement this method, we provide algorithm pseudocode and empirical experiments to demonstrate the efficiency. The algorithm presented in this study shares similarities with the greedy search method of DirectLiNGAM. However, it diverges in two key aspects: first, we leverage Functional Principal Component Analysis (FPCA) for data preprocessing, and second, our independence test considers multivariate relationships rather than univariate ones. This makes Func-LiNGAM straightforward to implement. For the purpose of this paper, we focus on providing a basic implementation without delving into enhancing search methods or other optimizations, as they are not the primary focus of our research. The whole algorithm is as Algorithm~\ref{algo1}.

Note that the $W$ means the sampled time points from one random function. As the intrinsically infinite-dimensional property of functional data, we need to approximate $W$ with efficient finite representation (FPCA with principal component number $M$ ($M\ll W$)). The number $M$ can be decided by the explained variance ratio ($95\%$ or $99\%$). To be simple, here we let all the $M$ of random functions be the same.
\begin{algorithm}
\caption{Func-LiNGAM (Can be regarded as vector-based DirectLiNGAM but with FPCA preprocessing.)}
\label{algo1}
\begin{algorithmic}[1]
\State \textbf{Input:} Each function has $W$ time points, then construct $Wp$-dimensional random vector ${f}$ ($W$: Full-time points) for $p$ functions, a set of its variable subscripts $U$ and a $Wp \times n$ data
matrix as $F$, initialize an ordered list of functions $K=\emptyset$ and $m:=1$; 
\State \textbf{Output:} Adjacent Matrix $\hat{T}\in\mathbb{R}^{p\times p}$
\State Use FPCA for finite approximating each random vector to make their dimensions from $Wp$ to $Mp$, where $M$ is the number of principal components.
\Repeat
\State 
(a) Perform least squares regressions of the approximating random vector $\hat{f}_{i}\in\mathbb{R}^M$ on  $\hat{f}_{j}\in\mathbb{R}^M$ for all $ i \in U \backslash K(i \neq j) $ and compute the residual vectors $ \mathbf{r}^{(j)} $ and the residual data matrix $ \mathbf{R}^{(j)} $ from the data matrix $ {F} $ for all $ j \in U \backslash K $. Find a variable $ \hat{f}_{m} $ that is most independent of its residuals:
$$
\hat{f}_{m}=\arg \min _{j \in U \backslash K} MI\left(\hat{f}_{j} ; U \backslash K\right),
$$
where $ MI $ is the independence measure such as mutual information or other measures.
\State (b) Append $ m $ to the end of $ K $.
\State (c) Let  $\hat{\mathbf{f}}:=\mathbf{r}^{(m)}, \hat{F}:=\mathbf{R}^{(m)} $.
\Until{$p-1$ subscripts are appended to $K$}
\State Append the remaining variable to the end of $ K $.
\State Construct a strictly lower triangular matrix $\hat{T} $ by following the order in $ K $, and estimate the connection strengths $ \hat{T}_{i j} $ by using least squares regression in this paper.
\end{algorithmic}
\end{algorithm}

\section{Experiment}\label{s5}
To validate our method, we conducted comprehensive experiments using simulated data, as shown in Table~\ref{tab:my_label}. We observed an improvement in performance as the sample size increased across multiple functions. Notably, precision decreased monotonically and Structural Hamming Distance (SHD) increased monotonically as the number of functions ($p$) grew. Our data generation process, following the settings in \cite{FGM}, involved $ n \times p $ random functions, defined as:
\begin{equation}
    \begin{aligned}
         X_{ij}(t)&=\mathbf{\phi}(t)^{T} \delta_{ij}
    \end{aligned}
\end{equation}
where $i$ represents the $i_{th}$ sample ($i=1,\dots, n$), and $j$ denotes the $j_{th}$ random vector. The vector $\delta_{ij}\in\mathbb{R}^{5}$ can be an arbitrary non-Gaussian random vector. Here we generated these by first creating random vectors $q_{ij}\sim\mathcal{N}(0,I_5)$, then we square each element of the vector to get $\delta_{ij}$. The five-dimensional Fourier basis $ \mathbf{\phi}(t) $ was also used.
{We modeled the causal relationships in $\delta_i$ as follows:
\begin{equation}
    \begin{aligned}
        \delta_{i0}=\epsilon_{0},\quad\delta_{i1}=B_{1,0}\delta_{i0}+\epsilon_{1},\quad\dots,\quad
        \delta_{ip}=B_{p,p-1}\delta_{i(p-1)}+\epsilon_{p}
    \end{aligned}
\end{equation}
where $u_l\sim\mathcal{N}(0,I_5)$, then we square each element of the vector to get $\epsilon_l$. To be simple, we set $B_{l,l-1}=I_5,l=1,\dots,p$. 
The sample size is $ n=\{100, 200, 300, 700 \}$, $p=\{5,10,20,30,50,70\}$, and the observed values, $ g_{ij}(t_k) $, follow  
$$
	g_{ij}(t_k)=X_{ij}\left(t_{k}\right)+e_{i j k},
$$ 
where $e_{ijk}$ is derived from the square of the random variable $q_{ijk}$, where $ q_{ijk} \sim \mathcal{N}(0, 0.25) $. Specifically, $ e_{ijk} = q_{ijk}^2 $. Due to the squaring of a normally distributed variable with a variance of $0.25$, the resulting distribution of $ e_{ijk} $ can be described as a Gamma distribution with a shape parameter of \( \frac{1}{2} \) and a scale parameter of \( 0.5 \), applicable for \( i = 1, \dots, n \) and \( j = 1, \dots, p \).} Every random function is sampled at $ W=1000 $ equidistant time points, $ 0=t_{1}, \ldots, t_{1000}=1 $. 
		
We employ B-spline bases as a fitting technique for each random function instead of the Fourier basis to represent the actual data accurately. B-spline bases offer more flexibility and can capture the complex shapes and patterns present in the data. After fitting the random functions with B-spline bases, we calculate each random function's estimated principal component scores. These scores are derived from the basis coefficients, with the number of calculated principal component scores limited to the first $M$ components ($M \leq W$). The choice of $M$ allows us to control the dimensionality of the data representation, providing a balance between capturing the most important variability in the data and minimizing computational complexity. By calculating these estimated principal component scores, we obtain a concise representation of the data that encapsulates its essential characteristics while reducing its dimensionality. This approach allows for efficient analysis and interpretation of the random functions within the context of our methodology. We set $M=5$ ($99\%$ explained variance ratio) for the B-spline. Cross-validation can also obtain the optimal $M$. However, we set the parameters to ensure they maintain as much information as possible. 
We evaluate the Func-LiNGAM with Precision, Recall ratio, F1-score, and SHD (Structural Hamming Distance in \cite{SHD}) in {$50$} trials as Table~\ref{tab:my_label}. The smaller the SHD, the better the performance. To clarify, our objective is to demonstrate an implementation example rather than to propose a superior algorithm through comparison.

\begin{table}
    \centering
    \begin{center}
\scalebox{0.77}{
\begin{tabular}{llllllllllll}
\hline
\multirow{2}{*}{Data size}&\multirow{2}{*}{ Metrics  }&\multicolumn{6}{c}{Various number of functions (mean $\pm$ standard deviation)  }\\
\cline{3-8}
&&\multicolumn{1}{c}{$p=5$}&\multicolumn{1}{c}{$p=10$}&\multicolumn{1}{c}{$p=20$}&\multicolumn{1}{c}{$p=30$}&\multicolumn{1}{c}{$p=50$}&\multicolumn{1}{c}{$p=70$}\\
\cline{1-8}
\multirow{4}{*}{$n=100$}&Precision&{$0.76\pm0.14$}&\multicolumn{1}{c}{$0.64\pm0.10$}&{$0.57\pm0.09$}&{$0.40\pm0.06$}&{$0.30\pm0.04$}&{$0.25\pm0.03$}\\
&Recall&{$0.99\pm0.04$}&{$0.95\pm0.0$}&{$0.90\pm0.07$}&{$0.75\pm0.07$}&{$0.65\pm0.04$}&{$0.59\pm0.04$}\\
&F1&{$0.85\pm0.10$}&{$0.76\pm0.08$}&{$0.70\pm0.09$}&{$0.52\pm0.07$}&{$0.41\pm0.05$}&{$0.35\pm0.03$}\\
&SHD&{$1.40\pm0.95$}&{$5.03\pm1.91$}&{$13.47\pm4.17$}&{$33.47\pm6.56$}&{$74.73\pm9.86$}&{$119.47\pm10.70$}\\
\cline{1-8}
\multirow{4}{*}{$n=200$}&Precision&{$0.83\pm0.14$}&\multicolumn{1}{c}{$0.76\pm0.29$}&{$0.72\pm0.07$}&{$0.70\pm0.06$}&{$0.54\pm0.05$}&{$0.46\pm0.07$}\\
&Recall&{$1.00\pm0.00$}&\multicolumn{1}{c}{$0.80\pm0.24$}&{$0.99\pm0.01$}&{$0.97\pm0.03$}&{$0.88\pm0.03$}&{$0.81\pm0.07$}\\
&F1&{$0.90\pm0.08$}&\multicolumn{1}{c}{$0.78\pm0.27$}&{$0.83\pm0.05$}&{$0.81\pm0.05$}&{$0.67\pm0.05$}&{$0.59\pm0.07$}\\
&SHD&{$0.97\pm0.91$}&\multicolumn{1}{c}{$3.63\pm4.58$}&{$7.70\pm2.35$}&{$12.53\pm3.36$}&{$37.03\pm6.60$}&{$66.20\pm12.79$}\\
\cline{1-8}
\multirow{4}{*}{$n=300$}&Precision&{${0.85\pm0.13}$}&\multicolumn{1}{c}{$0.79\pm0.28$}&{${0.75\pm0.07}$}&{$0.74\pm0.05$}&{$0.70\pm0.05$}&{$0.60\pm0.04$}\\
&Recall&{${1.00\pm0.00}$}&\multicolumn{1}{c}{$0.84\pm0.23$}&{$1.00\pm0.00$}&{$0.99\pm0.01$}&{$0.99\pm0.01$}&{$0.93\pm0.03$}\\
&F1&{${0.92\pm0.08}$}&\multicolumn{1}{c}{$0.81\pm0.26$}&{${0.86\pm0.05}$}&{$0.85\pm0.03$}&{$0.82\pm0.03$}&{$0.73\pm0.04$}\\
&SHD&{${0.80\pm0.75}$}&\multicolumn{1}{c}{$3.17\pm4.43$}&{${6.57\pm2.50}$}&{$10.27\pm2.41$}&{$21.27\pm4.36$}&\multicolumn{1}{l}{$42.90\pm6.25$}\\
\cline{1-8}
\multirow{4}{*}{$n=700$}&Precision&{${0.92\pm0.10}$}&\multicolumn{1}{c}{$0.81\pm0.08$}&{${0.80\pm0.07}$}&{$0.78\pm0.05$}&{$0.74\pm0.03$}&{$0.70\pm0.02$}\\
&Recall&{${1.00\pm0.00}$}&\multicolumn{1}{c}{$1.00\pm0.00$}&{$1.00\pm0.00$}&{$1.00\pm0.00$}&{$1.00\pm0.00$}&{$0.95\pm0.05$}\\
&F1&{${0.96\pm0.06}$}&\multicolumn{1}{c}{$0.88\pm0.05$}&{${0.88\pm0.04}$}&{$0.87\pm0.03$}&{$0.85\pm0.02$}&{$0.83\pm0.05$}\\
&SHD&{${0.40\pm0.55}$}&\multicolumn{1}{c}{$2.50\pm1.20$}&{${4.96\pm2.06}$}&{$8.80\pm2.34$}&{$17.40\pm2.97$}&{$32.70\pm4.37$}\\

\cline{1-8}

\end{tabular}}
\end{center}

    \caption{Evaluation of Func-LiNGAM with various number $p$ of functions. The causal graph is as $f_1\rightarrow f_2\rightarrow \dots \rightarrow f_p$ (50 trials).}
    \label{tab:my_label}
\end{table}


\section{Actual Data}\label{s6}

This section demonstrates the application of the proposed approach to analyzing brain connectomes for functional magnetic resonance imaging (fMRI) data. The fMRI data \citep{OpenNeuro} is preprocessed by downsampling it to a resolution of 4mm, with a repetition time (TR) of 2 seconds. This data consists of 155 subjects $(n=155)$, 168 time points {$(W=168)$}, and 17 parcels $(p=17)$.  During the study, 155 participants took part in the fMRI scans. Among them, 122 participants were children, 33 were adults. The participants were instructed to watch a short animated movie that aimed to evoke various mental states and physical sensations about the characters depicted in the movie. Our objective is to investigate the causal relationships between various brain regions when individuals watch the short film, regardless of age. To check the Gaussianity of the observed functions, we performed the Shapiro–Wilk normality test \citep{S-W} on $p = 17$ parcels at each $W = 168$ time point. The null hypothesis (i.e., the observations are marginally Gaussian) was rejected
for many combinations of scalp position and time point, and therefore, the non-Gaussianity
of the proposed model is deemed appropriate.
Next, we estimate the adjacency matrix between the parcels with the number of principal components $M=5$. The adjacency matrix reveals the presence of connections between specific parcel pairs. 
To visualize the brain connectivity and causal relationships, we present a 2D graph using the \verb|Nilearn| Python package and a 3D graph using the BrainNet Viewer \citep{Brain} {(Fig.~\ref{fig:6})}.
\begin{figure}[!ht]
	\centering

\includegraphics[width=0.6\textwidth,height=0.25\textwidth]{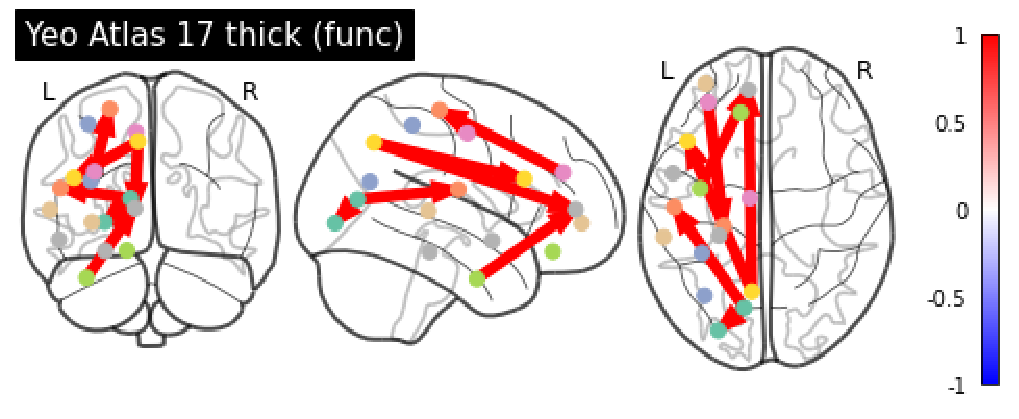}\hfill
\includegraphics[width=0.38\textwidth,height=0.25\textwidth]{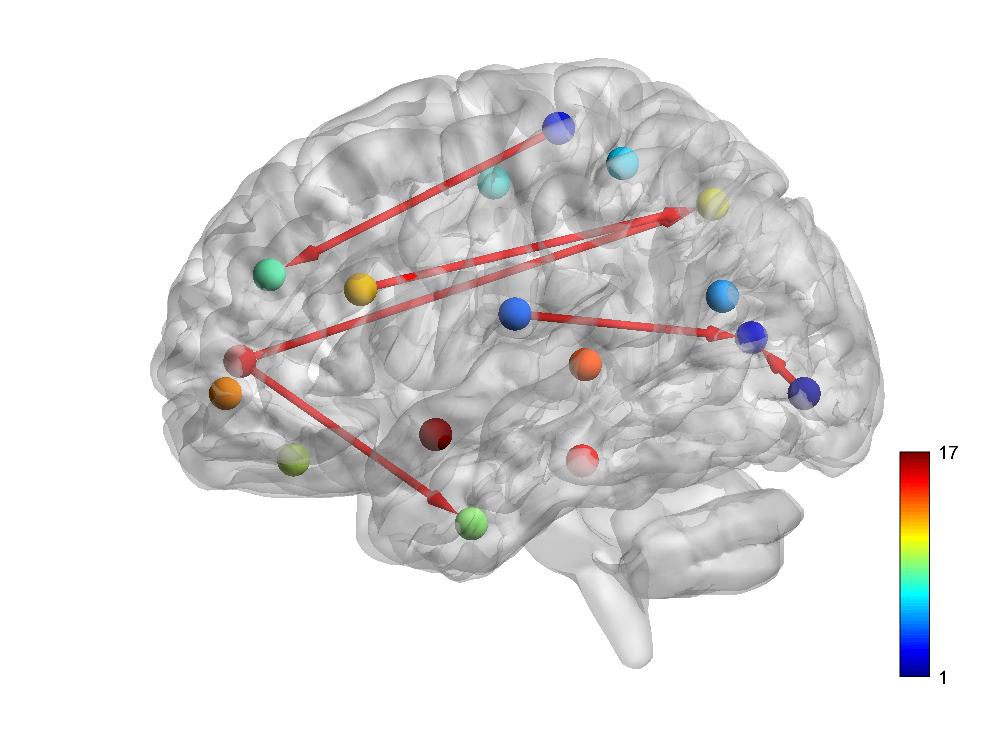}\hfill
		\caption{Brain Connectivity Graphs (Left: 2D , Right: 3D).}
  \label{fig:6}
\end{figure}

\section{Conclusion}\label{s7}
We have introduced a novel framework called Func-LiNGAM, which aims to identify causal relationships among random functions. For the theoretical foundation of Func-LiNGAM, we have proven the identifiability of both non-Gaussian random vectors (Theorem \ref{t1}) and non-Gaussian processes (Theorem \ref{t6}). Additionally, we have proposed a method to approximate random functions using random vectors based on Functional Principal Component Analysis (FPCA).
Empirically, we demonstrate that the proposed procedure of Func-LiNGAM achieves accurate and efficient identification of causal orders among non-Gaussian random functions. Furthermore, we have preliminarily applied Func-LiNGAM to analyze brain connectivity using fMRI data. Our framework combines theoretical advancements with practical applications, showcasing its effectiveness in identifying causal relationships among random functions and its potential for various domains, such as brain connectivity.

\bibliography{references}   


\end{document}